\documentclass{article}

\usepackage{blindtext}
\usepackage[a4paper, total={6in, 8in}]{geometry}

\usepackage[table,x11names]{xcolor} 

\usepackage{graphicx} 

\usepackage{graphicx} 

\usepackage{amsmath} 

\usepackage{mathtools}

\usepackage{hyperref} 
\usepackage{amssymb} 

\newenvironment{proof}{\paragraph{Proof:}}{\hfill$\square$}

\usepackage{xcolor}
\hypersetup{
    colorlinks,
    linkcolor={red!50!black},
    citecolor={blue!50!black},
    urlcolor={blue!80!black}
}

\usepackage[backend=biber, maxnames=99]{biblatex}

\usepackage{float}

\newtheorem{theorem}{Theorem}[section]
\newtheorem{corollary}{Corollary}[theorem]
\newtheorem{lemma}[theorem]{Lemma}

\usepackage{mathtools}

\usepackage{amsmath,amssymb}

\usepackage{tikz}
\usepackage{pgfplots} 
\pgfplotsset{compat=newest}
\pgfplotsset{plot coordinates/math parser=false}
\pgfplotsset{
    every non boxed x axis/.style={
        xtick align=center,
        enlarge x limits=true,
        x axis line style={line width=0.8pt, -latex}
},
    every boxed x axis/.style={}, enlargelimits=false
}
\pgfplotsset{
    every non boxed y axis/.style={
        ytick align=center,
        enlarge y limits=true,
        y axis line style={line width=0.8pt, -latex}
},
    every boxed y axis/.style={}, enlargelimits=false
}
\usetikzlibrary{
   arrows.meta,
  intersections,
}

\newcommand{\Prob}{\mathbb{P}}

\title{A distribution-free valid p-value for finite samples of bounded random variables}
\author{Joaquin Alvarez\footnote{ joaquin.alvarez@itam.mx}}

\date{
ITAM, Mexico
}

\begin{document}
\maketitle
\begin{abstract}
We build a valid p-value based on the Pelekis-Ramon-Wang inequality for bounded random variables introduced in \cite{pelekis2015bernsteinhoeffding}. The motivation behind this work is the  calibration of predictive algorithms in a distribution-free setting. The super-uniform p-value is tighter than Hoeffding's and Bentkus's alternatives \emph{in certain regions}. Even though we are motivated by a calibration setting in a machine learning context, the ideas presented in this work are also relevant in classical statistical inference. Furthermore, we compare the power of a collection of valid p-values for bounded losses, which are presented in \cite{angelopoulos2022learn}. 
\end{abstract}

\maketitle

\section{Introduction}

In the context of hypothesis testing problems, valid p-values (also known in the literature as super-uniform p-values) offer test statistics that can be used controlling any user-specified significance level. In particular, valid p-values are useful because many multiple testing algorithms use them as an input \cite{Barber_2015, dmitrienko2010multiple}. Moreover, recent literature has introduced frameworks where valid p-values play a key role to provide rigorous statistical guarantees and uncertainty quantification for a variety of machine learning models \cite{liang2022integrative, angelopoulos2022learn, Bates_2023, angelopoulos2023predictionpowered}.

In this work we build upon a concentration inequality from \cite{pelekis2015bernsteinhoeffding}, and construct a valid p-value for the calibration of predictive algorithms and multiple testing algorithms. We call this valid p-value the  PRW valid p-value (due to Pelekis, Ramon and Wang).

Our work is motivated by applications in predictive inference, particularly, it is motivated by recent work on distribution-free uncertainty quantification for black-box algorithms \cite{angelopoulos2022learn, bates2021distributionfree}. On a broader perspective, this work relies on the idea of leveraging concentration inequalities and converting them to obtain valid p-values \cite{bates2021distributionfree}. Our contribution is mainly concerned  with the details of such strategy for a specific concentration inequality. There are variety of concentration inequalities \cite{sridharan2002gentle} for which these kind of ideas can be replicated with adequate modifications and considerations.

Consider a collection of $n$ independent and identically distributed random variables bounded in $[0,1]$, $L_1,L_2,\dots,L_n$. We denote $R\coloneq\mathbb{E}(L_i)\in (0,1)$. We use $\hat{R}$ to denote the sample mean, namely, $\hat{R}\coloneqq \frac{1}{n}\sum_{i=1}^{n}L_i$. For example,  the random sample $\{L_i\}_{i=1}^{n}$ may be associated to the losses of a machine learning model in some calibration set, as in \cite{angelopoulos2022learn}. In such setting we refer to $\hat{R}$ as the empirical risk and $R$ as the theoretical risk. We assume that the latter one is unknown. We may be interested in a hypothesis testing problem:


\begin{equation}
    H_0:R>\alpha\text{ vs. }H_1: R\leq \alpha. \label{hyp}
\end{equation}
Our objective is to determine if there is statistical significance that a pre-trained predictive algorithm producing losses $\{L_i\}_{i=1}^{n}$  on a calibration set has a risk (expected loss) below some specified level $\alpha\in (0,1)$.

The motivation for this work from the perspective of classical statistical inference is a setting where we would like to do a hypothesis test for the mean of a collection of $n$ i.i.d. random variables  as in \ref{hyp}, where we do not know any other information about the  generating distribution other than the fact that its support is a subset of the interval $[0,1]$. The random variables can be either discrete of continuous.

\section{Main results}
Let us begin considering Theorem 1.8 from \cite{pelekis2015bernsteinhoeffding}, for the special case where the involved random variables are all identically distributed.
\begin{theorem}

Let $X_1,\dots,X_n$ be independent and identically distributed (i.i.d.) random variables such that $0\leq X_i\leq 1$ and $p\coloneqq \mathbb{E}[X_i]$. Let $n\in \mathbb{N}$, $p\in(0,1)$ and $Bin(n,p)$ denote a Binomial random variable with parameters $(n,p)$. Then for any positive integer $t$ such that $np<t\leq n$,\footnote{The authors state it with $t<n$, but we can get convinced that the bound also holds for $t=n$. See the appendix for a further discussion.}

$$ \mathbb{P}\Big(\sum_{i=1}^n X_i\geq t\Big)\leq \frac{t-tp}{t-np}\Prob\{Bin(n,p)\geq t\}.$$
\end{theorem}
Just as the authors of \cite{pelekis2015bernsteinhoeffding} point out, for certain values of $p,n,\text{ and } t$ this upper bound is smaller than the one provided by Bentkus in his celebrated  inequality \cite{Bentkus_2004}. We will verify those observations in the valid p-value formulation in the last section of this paper.

A classical idea for many  concentration inequalities is that the results that hold for a lower tail also hold for the upper tail and vice versa. This is the case for this inequality.

\begin{corollary}(Lower tail version)\label{coro}
Let $L_1,\dots, L_n$ be i.i.d. bounded random variables in $[0,1]$, and $R\coloneqq \mathbb{E}(L_i)\in (0,1)$.
Then for any integer $k\in [0,nR)$,
$$\Prob\Big( \sum_{i=1}^{n}L_i\leq k \Big)\leq \frac{R(n-k)}{nR-k}\Prob\Big(Bin(n,R)\leq k \Big).$$
\end{corollary}

\begin{proof}
Let $k\in [0,nR)$. Consider a change of variables, $X_i\coloneqq 1-L_i$, and we take $p\coloneq1-R$, which is the mean of each $X_i$; notice that $p\in (0,1)$ and $0\leq X_i\leq 1$ for all $i \in \{ 1,\dots, n\}$. Also notice that  $k\in [0,nR)$ implies $n-k \in (np,n]$.

\begin{equation*}
    \begin{split}
       \Prob\{ \sum_{i=1}^{n}L_i\leq k \}&=\Prob\{ \sum_{i=1}^{n}(1-X_i)\leq k \}\\
       &=\Prob\{ \sum_{i=1}^{n}X_i\geq n-k\}\\
       &\leq \frac{n-k-(n-k)(1-R)}{n-k-n(1-R)} \Prob\{Bin(n, 1-R)\geq n-k\}\\
       &=\frac{(n-k)R}{nR-k} \Prob\{n-Bin(n, 1-R)\leq k\}\\
       &=\frac{(n-k)R}{nR-k} \Prob\{Bin(n, R)\leq k\}.
    \end{split}
\end{equation*}
\end{proof}

We can reformulate this corollary in the following way:

\begin{equation} 
    \begin{split}
      \Prob\{\frac{1}{n}\sum_{i=1}^{n}L_i\leq t \}  &=\Prob\{\sum_{i=1}^{n}L_i\leq nt\}\\
      &\leq \Prob\{\sum_{i=1}^{n}L_i\leq \lceil nt \rceil\} \\
      &\leq \frac{R(n-\lceil nt \rceil)}{nR-\lceil nt \rceil}\Prob\{Bin(n,R)\leq \lceil nt \rceil \} ,
    \end{split}
    \label{equu}
\end{equation}

for values of $t\in [0,R)$ such that $0\leq \lceil nt \rceil<nR$. We applied Corollary \ref{coro}, and we use the fact that any CDF is non-decreasing. A visual representation of the function $t\mapsto \lceil nt \rceil$ is helpful to understand the upper bound for $t$. We present such representation in Figure 1.

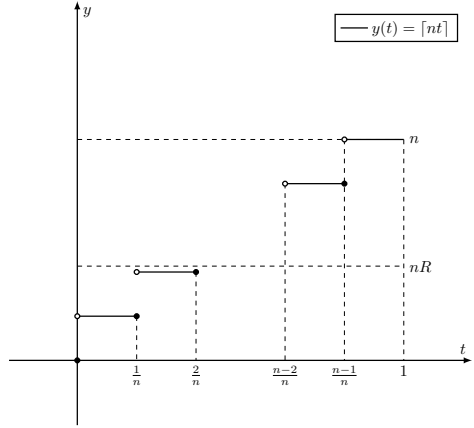
\begin{figure}[H]
\centering
\begin{tikzpicture}[scale=0.60]
\begin{axis}[width=5in,axis equal image,
    smooth,thick,
      axis lines=center,
    xmin=-1,xmax=12,
    xlabel=$t$,ylabel=$y$,
    ymin=-1,ymax=11,
    xtick={\empty},ytick={\empty},
    legend style={cells={anchor=west}},
    legend pos= north east
]

\addlegendentry{$y(t)=\lceil nt\rceil$}

\addplot [
    domain=0:2, 
    samples=100, 
    color=black,
]
{1.5};

\addplot [
    domain=2:4, 
    samples=100, 
    color=black,
]
{3};

\addplot [
    domain=7:9, 
    samples=100, 
    color=black,
    ]
{6}; 

\addplot [
    domain=9:11, 
    samples=100, 
    color=black,
    ]
{7.5}; 

\draw[dashed, line width=0.15mm] (2,1.5) -- (2,0) node[below] {$\frac{1}{n}$};
\draw[dashed, line width=0.15mm] (4,3) -- (4,0) node[below] {$\frac{2}{n}$};

\draw[dashed, line width=0.15mm] (7, 6) -- (7,0) node[below] {$\frac{n-2}{n}$};

\draw[dashed, line width=0.15mm] (9, 7.5) -- (9,0) node[below] {$\frac{n-1}{n}$};

\draw[dashed, line width=0.15mm] (0, 3.2) -- (11,3.2) node[right] {$nR$};

\draw[dashed, line width=0.15mm] (11, 7.5) -- (11,0) node[below] {$1$};

\draw[dashed, line width=0.15mm] (0, 7.5) -- (11,7.5) node[right] {$n$};
\draw [fill=black] (0,0) circle [radius=1.5pt];
\draw [fill=white] (0,1.5) circle [radius=1.5pt];

\draw [fill=black] (2,1.5) circle [radius=1.5pt];
\draw [fill=white] (2,3) circle [radius=1.5pt];
\draw [fill=black] (4,3) circle [radius=1.5pt];
\draw [fill=white] (7,6) circle [radius=1.5pt];
\draw [fill=white] (9,7.5) circle [radius=1.5pt];
\draw [fill=black] (9,6) circle [radius=1.5pt];
\end{axis}
\end{tikzpicture}
\caption{\normalfont \footnotesize A graph of $\lceil nt\rceil$ in $[0,1]$, with some hypothetical critical values to make sure that the upper bound is well defined. This figure motivates the definition of $\gamma(R)$. The function $\lceil nt\rceil$ takes positive jumps at $\{0,\frac{1}{n}, \frac{2}{n},\dots,  \frac{n-1}{n}\}$.}
\end{figure}

Given that $R\in (0,1)$, we have that $0<nR<n$ for any $n \in\mathbb{N}$. Hence, we take $$\gamma(R)\coloneqq \text{min}\{r \in \mathbb{N}: r\geq nR \}.$$ 

By the  Archimedian Property, the involved set is non-empty. Furthermore, given that the set is not empty, the Well Order Principle  guarantees that the minimum always exists. Note  the fact that $n$ is always an element in $\{r \in \mathbb{N}: r\geq nR \}$.

Moreover, we consider values of $t\in [0,R)$ such that $t\leq\frac{\gamma(R)-1}{n}$.\footnote{Initially we considered taking $t\leq \gamma(R)/n$ if $nR\notin \{0,1,\dots, n-1\}$. But since this could require to study separately the case when $nR\in \{0,1,\dots, n-1\}$  we decided to stick with the more conservative definition that considers both cases simultaneously.} 
Note that by definition, $\frac{\gamma(R)-1}{n}<R$. This election of $\gamma(R)$ guarantees that the upper bound for \ref{equu} is well defined. Equivalently, what the definition of $\gamma(R)$ helps us ensure is that 
$$\Big[0,\frac{\gamma(R)-1}{n}\Big]\subset \{t\in [0,R): 0\leq \lceil nt\rceil<nR \},$$
because by definition, $\gamma(R)-1\notin\{r \in \mathbb{N}: r\geq nR \}$, so $$t\in \Big[0,\frac{\gamma(R)-1}{n}\Big] \implies nt\leq \gamma(R)-1<nR \implies \lceil nt\rceil \leq \gamma(R)-1<nR .$$

\subsection{Outline}

\begin{itemize}
    \item We will show that if we define $g(t;R)\coloneqq \frac{R(n-\lceil nt \rceil)}{nR-\lceil nt \rceil}\Prob\{Bin(n,R)\leq \lceil nt \rceil \}$ for all $t\in [0, \frac{\gamma(R)-1}{n})$ and $g(\frac{\gamma(R)-1}{n};R)\coloneqq \text{max}\{1, \frac{R(n-\gamma(R)+1)}{nR-\gamma(R)+1}\Prob\{Bin(n,R)\leq \gamma(R)-1\}\} $, then we obtain that 

$$g\Big(\text{min}\{\hat{R}, \frac{\gamma(\alpha)-1}{n}\};\alpha\Big)$$ is a super-uniform p-value to test $H_0:R>\alpha$. We will refer to it as PRW valid p-value. We will also provide some plots comparing this valid p-value to Bentkus's valid p-value and Hoeffding's tight super-uniform p-value, showing that the PRW valid p-value is tighter than both of them in some regions.  As an auxiliary step towards this goal we need intermediate considerations.

\item We will show that $g$ has the following properties. Firstly, $g(\cdot;R)$ is a non-decreasing function in its first argument for any fixed value of $R$, for values of $t\in [0,\frac{\gamma(R)-1}{n}]$. For any $R_1,R_2\in (0,1)$  and any  $t\in [0,\frac{\gamma(R_1)-1}{n})$, such that $R_1\leq R_2$, $g(t;R_2)\leq g(t;R_1)$. Moreover,  $t\in  [0,\frac{\gamma(R_1)-1}{n})$ implies $t\in [0,\frac{\gamma(R_2)-1}{n})$ so that we are always in the region where we've defined $g$. Finally, for any $\delta \in (0,1)$  such that $\delta>(1-R)^n$, we can define 
$$g^{-1}(\delta;R)\coloneqq \text{sup}\Big\{t\in \big[0,\frac{\gamma(R)-1}{n}\big]\quad|\quad g(t;R)\leq \delta\Big\}.$$
 And by the fact that $t\mapsto\lceil nt\rceil$ is left continuous and non-decreasing, such supremum is always a maximum and is well defined. Given that $g(0;R)=(1-R)^n$, the underlying set is always non-empty for values of $\delta \in  (0,1)$ such that $(1-R)^n <\delta$.

\end{itemize}

\begin{lemma} $g(\cdot; R):[0,\frac{\gamma(R)-1}{n}]\longrightarrow \mathbb{R}$ is non-decreasing for any fixed $R\in (0,1)$.
\end{lemma}
\begin{proof}
Let $R\in (0,1)$.

    Let $t_1,t_2\in [0,\frac{\gamma(R)-1}{n}]$ be such that $t_1<t_2$ and $k_j\coloneqq\lceil nt_j\rceil$ for $j=1,2$ such that $k_2=k_1+1$. 

In our first case, we consider $t_2<\frac{\gamma(R)-1}{n}$. Then

    \begin{equation*}
        \begin{split}
            g(t_1;R)<g(t_2;R)&\iff\frac{R(n-\lceil nt_1 \rceil)}{nR-\lceil nt_1 \rceil}\Prob\{Bin(n,R)\leq \lceil nt_1 \rceil \}<\frac{R(n-\lceil nt_2 \rceil)}{nR-\lceil nt_2 \rceil}\Prob\{Bin(n,R)\leq \lceil nt_2 \rceil \}\\
            &\iff \frac{R(n-k_1)}{nR-k_1}\Prob\{Bin(n,R)\leq k_1 \}<\frac{R(n-k_1-1)}{nR-k_1-1}\Prob\{Bin(n,R)\leq k_1+1 \}.\\
        \end{split}
    \end{equation*}
Separately we consider $\Prob\{Bin(n,R)\leq k_1 \}<\Prob\{Bin(n,R)\leq k_1+1 \}$, which is true by the properties of a CDF of any binomial random variable. Whereas

\begin{equation*}
    \begin{split}
      \frac{R(n-k_1)}{nR-k_1}&<\frac{R(n-k_1-1)}{nR-k_1-1}\iff
    R(n-k_1)(nR-k_1-1)<R(n-k_1-1)(nR-k_1)\\
    &\iff (n-k_1)(nR-k_1)-(n-k_1)<(n-k_1)(nR-k_1)-(nR-k_1)\\
    &\iff k_1-n<k_1-nR\\
    &\iff nR<n\\
    &\iff R<1.
    \end{split}
\end{equation*}
We conclude the desired result multiplying both quantities. Moreover this proves that the smallest value that $g(t;R)$ can attain for values of $t\in [0,\frac{\gamma(R)-1}{n}]$ is $g(0;R)=(1-R)^n$.
And if $t_2=\frac{\gamma(R)-1}{n}$, then the same reasoning proves that 
$$g(t_1;R)<\frac{R(n-\lceil nt_2 \rceil)}{nR-\lceil nt_2 \rceil}\Prob\{Bin(n,R)\leq \lceil nt_2 \rceil \}.$$
And by definition of $g(\frac{\gamma(R)-1}{n};R)$, $\frac{R(n-\lceil nt_2 \rceil)}{nR-\lceil nt_2 \rceil}\Prob\{Bin(n,R)\leq \lceil nt_2 \rceil \}\leq g(\frac{\gamma(R)-1}{n};R)$, so we can conclude that $g(t_1;R)<g(t_2;R)$ when $0\leq t_1<t_2=\frac{\gamma(R)-1}{n}$.
\end{proof}

\begin{lemma} $g(t; \cdot):(0,1)\longrightarrow \mathbb{R}$ is non-increasing in  $R\in (0,1)$ for any fixed value of $t\in [0,\frac{\gamma(R)-1}{n})$.
    
\end{lemma}

\begin{proof}
    Let $0<R_1<R_2<1$. Let $t\in [0,\frac{\gamma(R_1)-1}{n})$. We notice that $$0<R_1<R_2<1\implies \gamma(R_1)\leq\gamma(R_2),$$ by definition of 
    $$\gamma(R)\coloneqq \text{min}\{ r\in \mathbb{N}: r\geq nR\}\text{ for all } R\in (0,1).$$
We have that $\{ r\in \mathbb{N}: r\geq nR_2\}\subset\{ r\in \mathbb{N}: r\geq nR_1\}$. So $t\in [0,\frac{\gamma(R_1)-1}{n})\subseteq [0,\frac{\gamma(R_2)-1}{n})$.

Let $k\coloneqq \lceil nt\rceil$.  Then 

\begin{equation*}
  g(t; R_2)\leq g(t; R_1)\iff \frac{R_2(n-k)}{nR_2-k}\Prob\{Bin(n,R_2)\leq k \}\leq\frac{R_1(n-k)}{nR_1-k}\Prob\{Bin(n,R_1)\leq k \}.
\end{equation*}
Clearly, $\Prob\{Bin(n,R_2)\leq k \}\leq\Prob\{Bin(n,R_1)\leq k \}$. On the other hand, 
$\frac{R_2(n-k)}{nR_2-k}\leq\frac{R_1(n-k)}{nR_1-k}$ if and only if $(nR_1-k)R_2\leq R_1(nR_2-k)$. In turn, this occurs if and only if $kR_1\leq kR_2$, which is true since $k$ is a non-negative integer and $R_1<R_2$.
\end{proof}

With this context in mind, we introduce the following theorem, which is the main result of this work.

 \begin{theorem}
 Let $L_1,\dots, L_n$ be a random sample of random variables bounded in $[0,1]$ whose mean is $R\in (0,1)$. Let $\alpha \in (0,1)$.  Consider the hypothesis testing problem $H_0:R>\alpha \text{ vs. }H_1:R\leq \alpha$. Take $\Hat{R}\coloneqq\sum\limits_{i=1}^{n}L_i $. Then $g\Big(\text{min}\{\hat{R}, \frac{\gamma(\alpha)-1}{n}\};\alpha\Big)$ is a valid p-value to test $H_0:R>\alpha$.
 \end{theorem}

 \begin{proof}
     Let $\delta\in (0,1)$.

    Case 1: $\delta\geq(1-R)^n$.

    Then

\begin{equation*}
    \begin{split}
        \Prob_{H_0}\Big(g\Big(\text{min}\{\hat{R}, \frac{\gamma(\alpha)-1}{n}\};\alpha\Big)\leq \delta\Big)&=\Prob_{H_0}\Big(g\Big(\text{min}\{\hat{R}, \frac{\gamma(\alpha)-1}{n}\};\alpha\Big)\leq \delta, \hat{R}\geq \frac{\gamma(\alpha)-1}{n}\Big)\\
        &+\Prob_{H_0}\Big(g\Big(\text{min}\{\hat{R}, \frac{\gamma(\alpha)-1}{n}\};\alpha\Big)\leq \delta, \hat{R}< \frac{\gamma(\alpha)-1}{n}\Big),  
    \end{split}
\end{equation*}

where we applied the law of total probability.
The event inside the probability of the first term is equal to the empty set due to the fact that by definition $g(\frac{\gamma(\alpha)-1}{n};\alpha)\geq 1$.
So we only need to consider the second term. 

\begin{equation*}
    \begin{split}
        \Prob_{H_0}\Big(g\Big(\text{min}\{\hat{R}, \frac{\gamma(\alpha)-1}{n}\};\alpha\Big)\leq \delta\Big)&=\Prob_{H_0}\{g\big(\hat{R};\alpha\big)\leq \delta, \hat{R}< \frac{\gamma(\alpha)-1}{n}\}\\
        &\leq\Prob_{H_0}\{g\big(\hat{R};R\big)\leq \delta, \hat{R}< \frac{\gamma(\alpha)-1}{n}\},
    \end{split}
\end{equation*}

where we used that $g$ is non-increasing in its second argument for fixed values of its first argument.\footnote{Also, notice that $\hat{R}< \frac{\gamma(\alpha)-1}{n}$ and $H_0:R>\alpha$ implies that $\hat{R}< \frac{\gamma(R)-1}{n}$ so the last probability evaluates $g$ where it is well defined. In particular, $H_0\implies \frac{\gamma(R)-1}{n}\geq \frac{\gamma(\alpha)-1}{n} $.} Next we apply the definition of $g^{-1}$

\begin{equation*}
    \begin{split}
    \Prob_{H_0}\{g\big(\hat{R};R\big)\leq \delta, \hat{R}< \frac{\gamma(\alpha)-1}{n}\}&=\Prob_{H_0}\{\hat{R}\leq g^{-1}(\delta;R), \hat{R}< \frac{\gamma(\alpha)-1}{n}\}\\
    &\leq \Prob_{H_0}\{\hat{R}\leq g^{-1}(\delta;R)\}\\
    &\leq g(g^{-1}(\delta;R);R)\leq\delta.
    \end{split}
\end{equation*}

    Case 2: $\delta<(1-R)^n$.
 Then  it is straightforward that $$\Prob_{H_0}\Big(g\Big(\text{min}\{\hat{R}, \frac{\gamma(\alpha)-1}{n}\};\alpha\Big)\leq \delta\Big)=0\leq \delta,$$
because the smallest value that $g(t;\alpha)$ can take for values of $t$ such that  $t\in [0, \frac{\gamma(\alpha)-1}{n}]$ is $(1-\alpha)^n$ and under $H_0$, $(1-R)^n\leq(1-\alpha)^n$. Thus under $H_0$, $\delta<(1-R)^n$ implies $\delta<(1-\alpha)^n$.
 \end{proof}

\section{A comment on some FWER controlling algorithms}

The proposed valid p-value is particularly relevant for regions where the observed empirical risk is much smaller than the considered $\alpha$ in the formulated hypothesis test in \ref{hyp} as the plots in the Appendix show. This is useful information for the purposes of FWER controlling algorithms that incorporate prior knowledge of the experiments to be carried out, for example in the fixed sequence algorithm and in the fallback procedure \cite{dmitrienko2010multiple}. Note that our plots suggest that the PRW super-uniform p-value is not very powerful for values of the empirical risk that are smaller than $\alpha$ but close to it, when comparing it to Bentkus and Hoeffding's valid p-values. However, we can leverage the fact that the PRW valid p-value is very powerful for small observations of the empirical risk. We can leverage this fact when implementing a fixed sequence procedure or a fallback procedure. In the context of those algorithms, given the provided order for the null hypotheses,  it suggests that the observed valid p-values using the PRW valid p-value will be smaller than Bentkus's and Hoeffding's valid p-values for many of the first nulls, hence increasing the power of the algorithms.

\printbibliography
\newpage
\section*{Appendix}
We study some properties of $g^{-1}(\delta;R)\coloneqq \text{sup}\Big\{t\in \big[0,\frac{\gamma(R)-1}{n}\big]\quad|\quad g(t;R)\leq \delta\Big\}$, for values of $\delta \in [(1-R)^n,1)$. By definition, $(1-R)^n=g(0;R)$, thus $0\in \Big\{t\in \big[0,\frac{\gamma(R)-1}{n}\big]\quad|\quad g(t;R)\leq \delta\Big\}$. Moreover, since $\Big\{t\in \big[0,\frac{\gamma(R)-1}{n}\big]\quad|\quad g(t;R)\leq \delta\Big\}\subset \big[0,\frac{\gamma(R)-1}{n}\big]$, then the set of interest is bounded above and by he Axiom of Completeness, $g^{-1}(\delta;R)$ is defined. Furthermore,  $g^{-1}(\delta;R)\leq \text{sup}\big[0,\frac{\gamma(R)-1}{n}\big]=\frac{\gamma(R)-1}{n}$. And by definition $g(\frac{\gamma(R)-1}{n};R)\geq 1$, thus $\frac{\gamma(R)-1}{n}\notin \Big\{t\in \big[0,\frac{\gamma(R)-1}{n}\big]\quad|\quad g(t;R)\leq \delta\Big\}$. We want to show that $g^{-1}(\delta;R)\in \Big\{t\in \big[0,\frac{\gamma(R)-1}{n}\big]\quad|\quad g(t;R)\leq \delta\Big\}$. That is, $g^{-1}(\delta;R)$ is a maximum. By contradiction suppose that $g(g^{-1}(\delta;R);R)>\delta$. But given that for any fixed $R\in (0,1)$, $g(t;R)$ is continuous from the left and has a step function form, $\exists \epsilon>0$ sufficiently small such that $g(g^{-1}(\delta;R)-\epsilon;R)=g(g^{-1}(\delta;R);R)>\delta$. And given that $g(t;R)$ is non-decreasing, then $t\leq g^{-1}(\delta;R)-\epsilon$ for all $t\in \Big\{x\in \big[0,\frac{\gamma(R)-1}{n}\big]: g(x;R)\leq \delta\Big\}$. But this contradicts that $g^{-1}(\delta;R)$ is a supremum. We conclude that $g(g^{-1}(\delta;R);R)\leq \delta$.

\subsection*{Theorem 2.1 holds for $t=n$}
If the involved random variables are continuous, it is then immediate that $$ \mathbb{P}\Big(\sum_{i=1}^n X_i\geq t\Big)\leq \frac{t-tp}{t-np}\Prob\{Bin(n,p)\geq t\},$$
because this inequality for $t=n$ simplifies to
$$ \mathbb{P}\Big(\sum_{i=1}^n X_i\geq n\Big)=\mathbb{P}\Big(\bigcap_{i=1}^{n} \{X_i= 1\}\Big)=0\leq \Prob\{Bin(n,p)\geq n\}.$$
Whereas in the discrete case, consider $Y_i\overset{i.i.d.}{\sim} Bernoulli(p)$, for $i\in \{1,\dots, n\}$.

\begin{equation*}
    \begin{split}
     \mathbb{P}\Big(\sum_{i=1}^n X_i\geq n\Big)  &=\mathbb{P}\Big(\bigcap_{i=1}^{n} \{X_i= 1\}\Big)\\
     &=\Prob(X_1=1)^n .
    \end{split}
\end{equation*}
On the other hand,  the  desired upper bound is given by

\begin{equation*}
    \begin{split}
    \frac{n-np}{n-np}\Prob\{Bin(n,p)\geq n\}&=\Prob\Big(\bigcap_{i=1}^{n} \{Y_i= 1\}\Big)\\
    &=p^n.
    \end{split}
\end{equation*}

Hence the inequality holds in the discrete case if and only if $\Prob(X_1=1)\leq p$. By definition of $p\coloneq\mathbb{E}\{X_1\}$ the result follows:

\begin{equation*}
    \begin{split}
    \Prob(X_1=1)\leq \sum_{x\in S_{X_1}}x\Prob(X_1=x)=\sum_{x\in S_{X_1}-\{1\}}x\Prob(X_1=x)+1\cdot\Prob(X_1=1),
    \end{split}
\end{equation*}

where $S_{X_1}\subseteq [0,1]$ denotes the support of the $X_i$'s.

\begin{figure}[p]
  \centering
  \includegraphics[width=.48\textwidth]{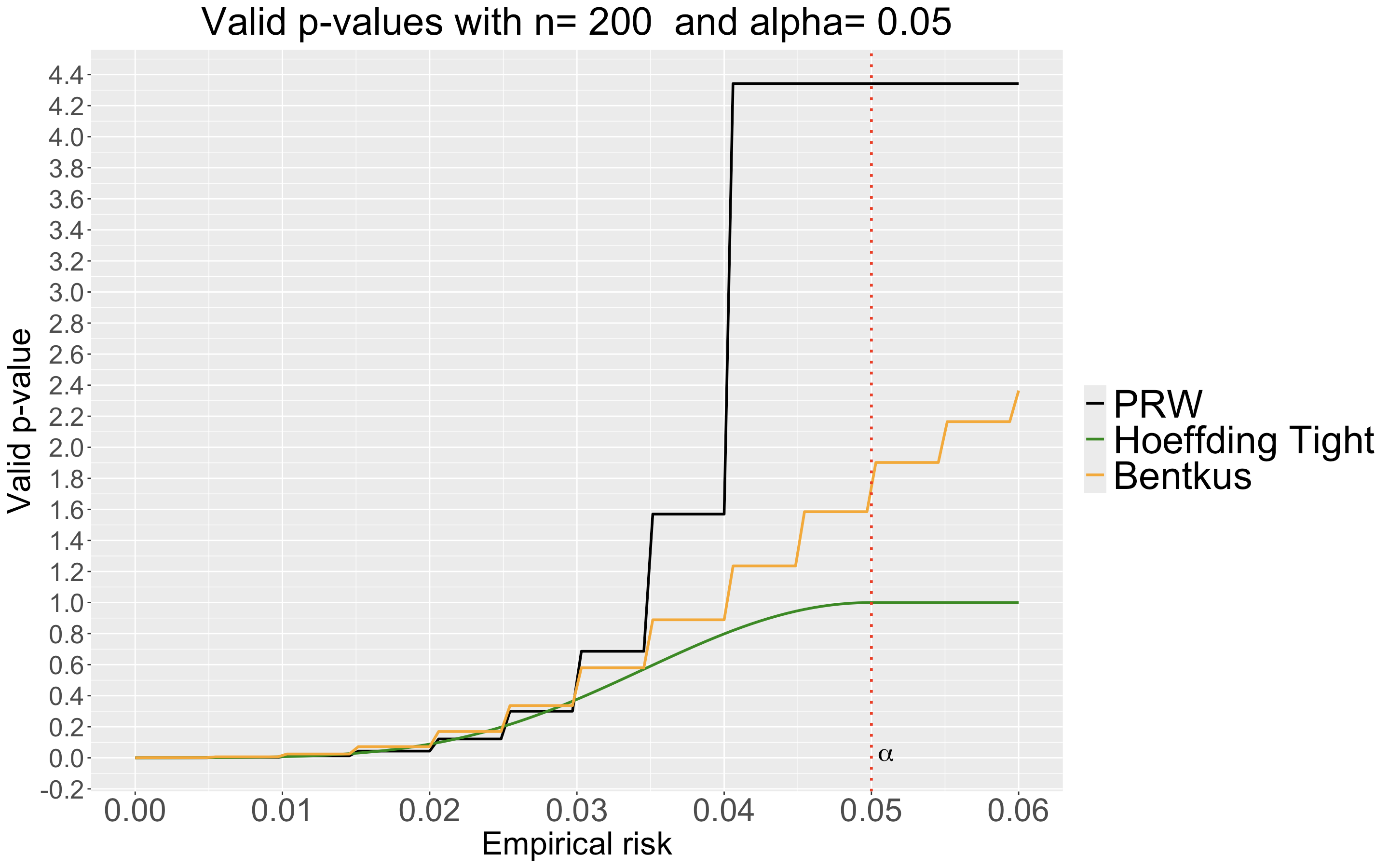}
  \hspace{.2cm}
  \includegraphics[width=.48\textwidth]{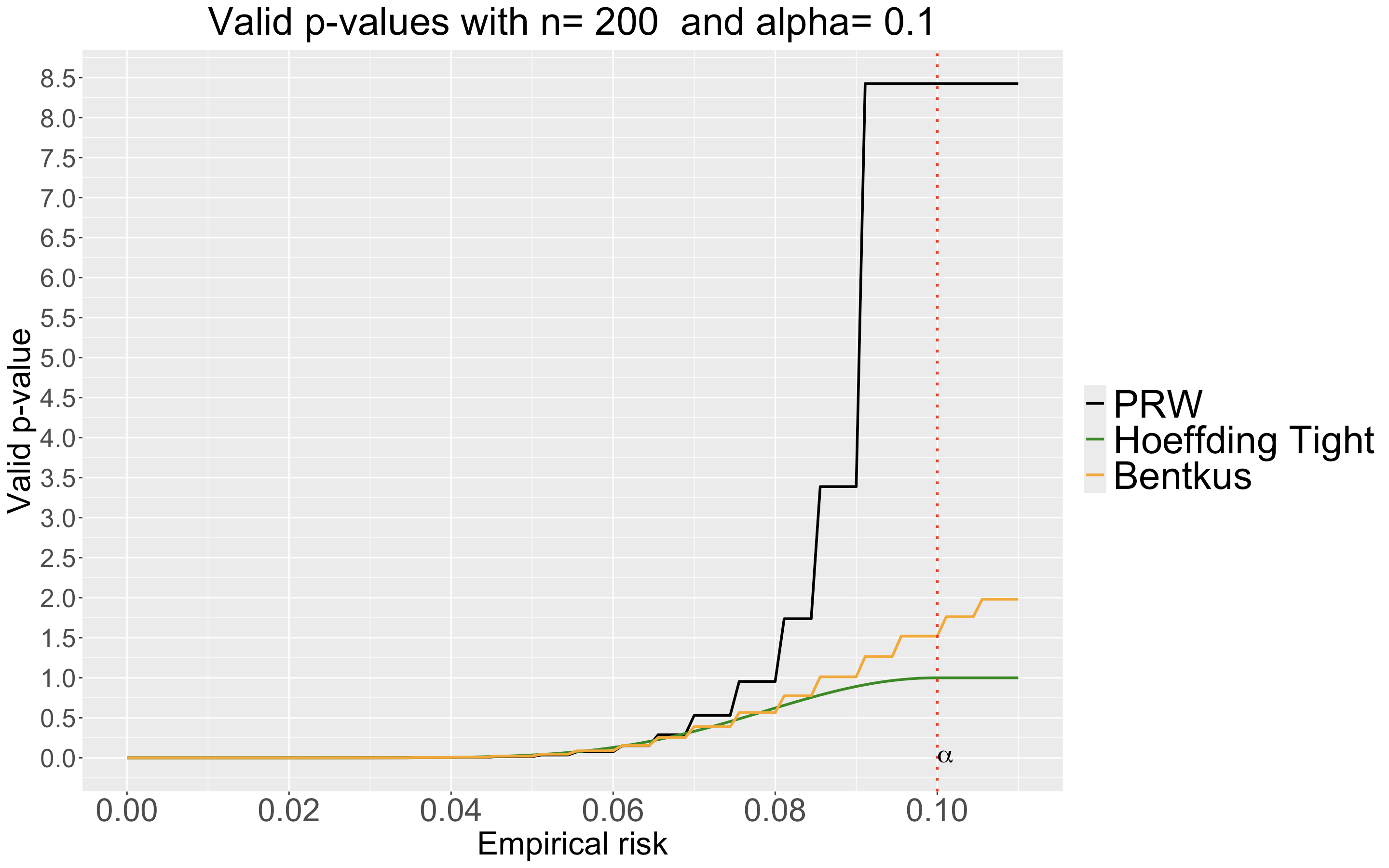}

  \vspace{1cm}
  
  \includegraphics[width=.48\textwidth]{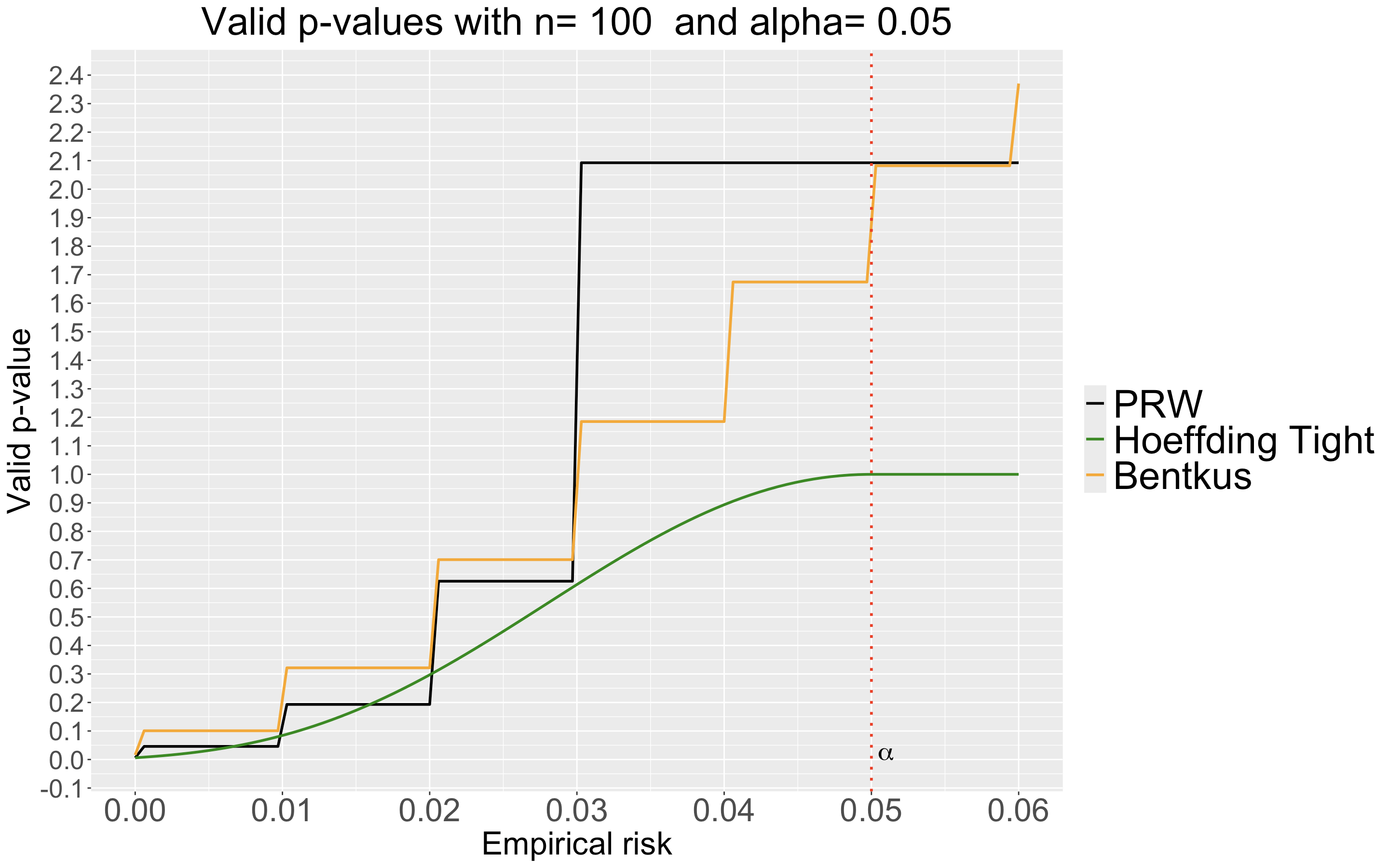}
  \hspace{.2cm}
  \includegraphics[width=.48\textwidth]{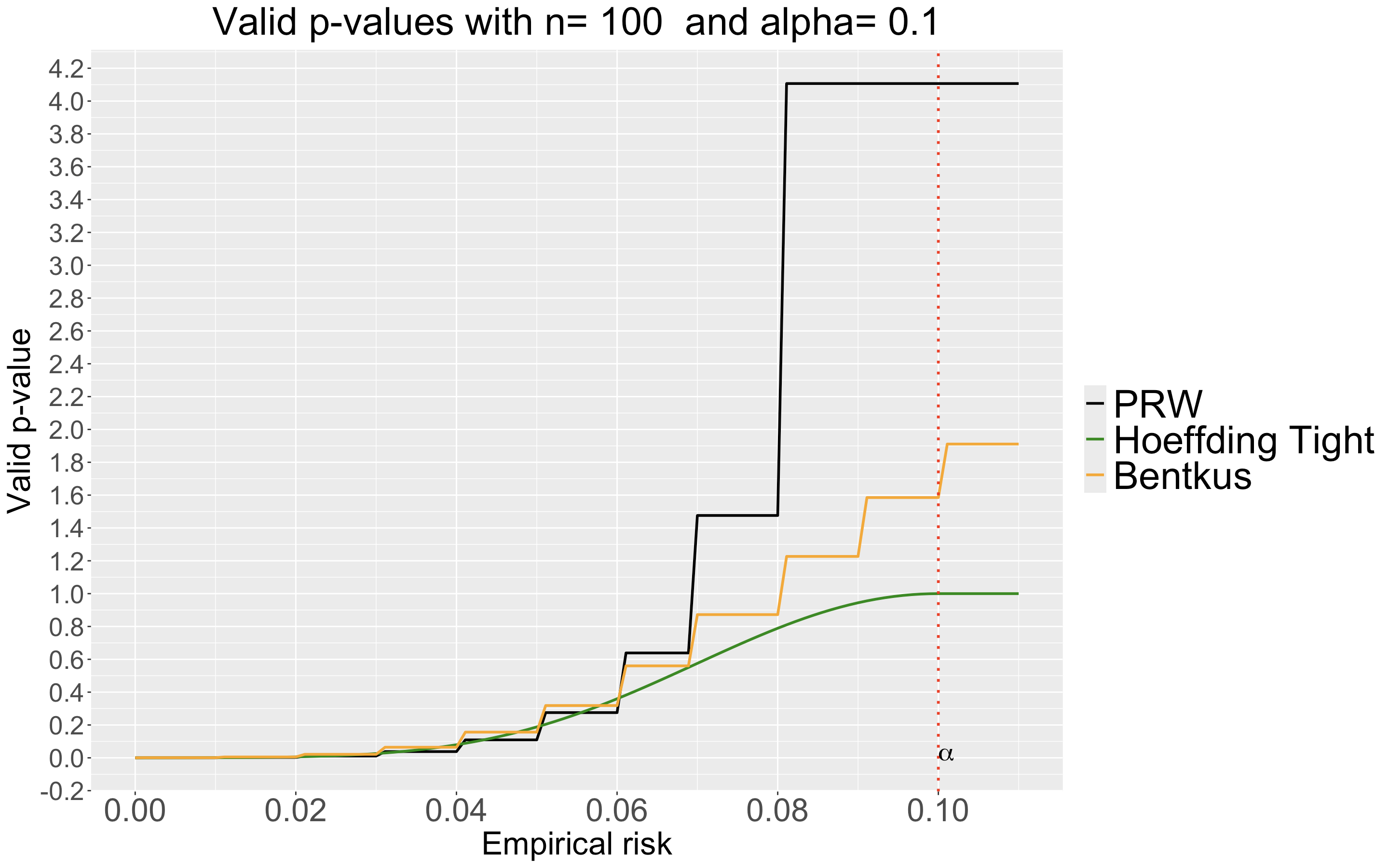}

  \vspace{1cm}

  \includegraphics[width=.48\textwidth]{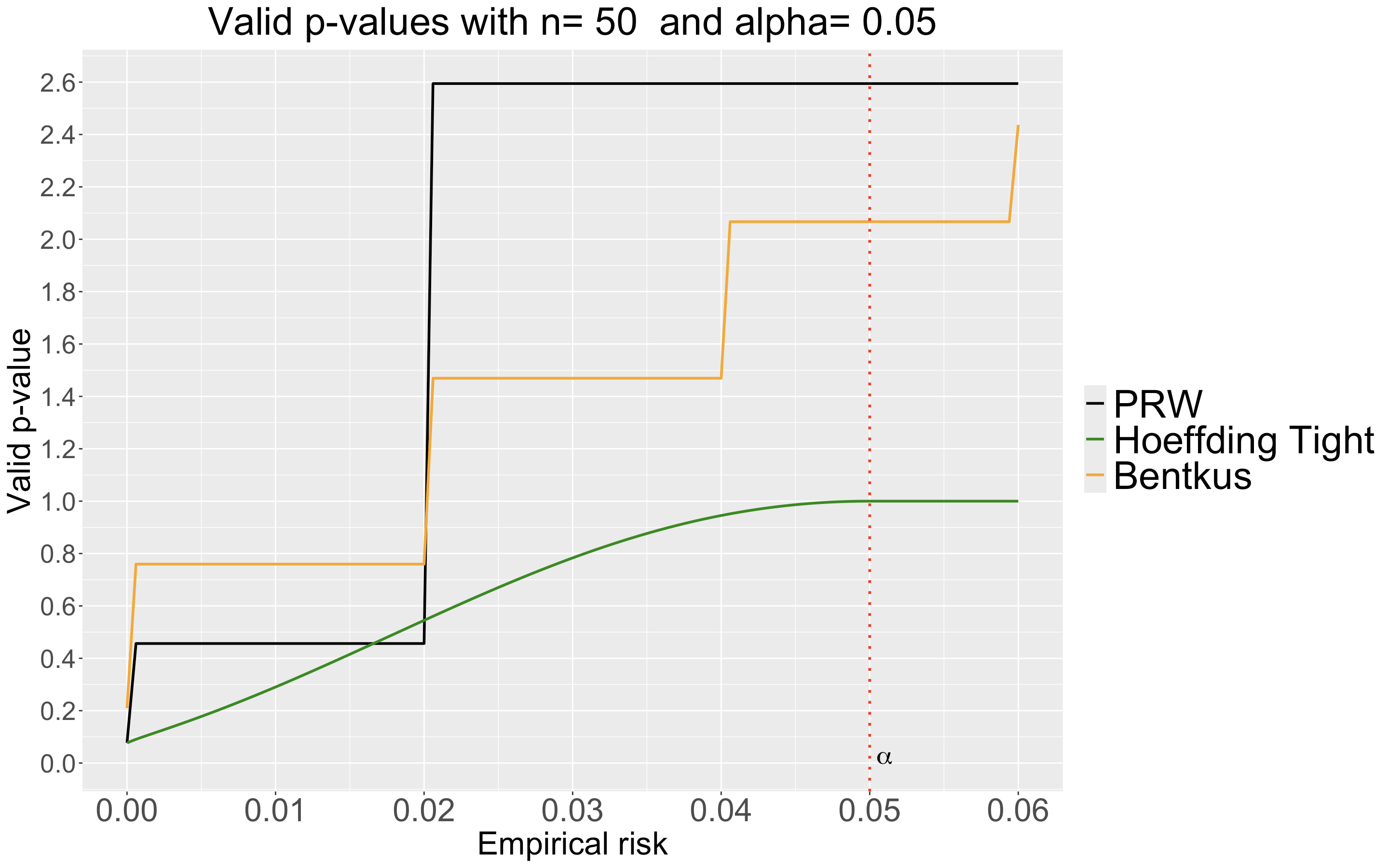}
  \hspace{.2cm}
  \includegraphics[width=.48\textwidth]{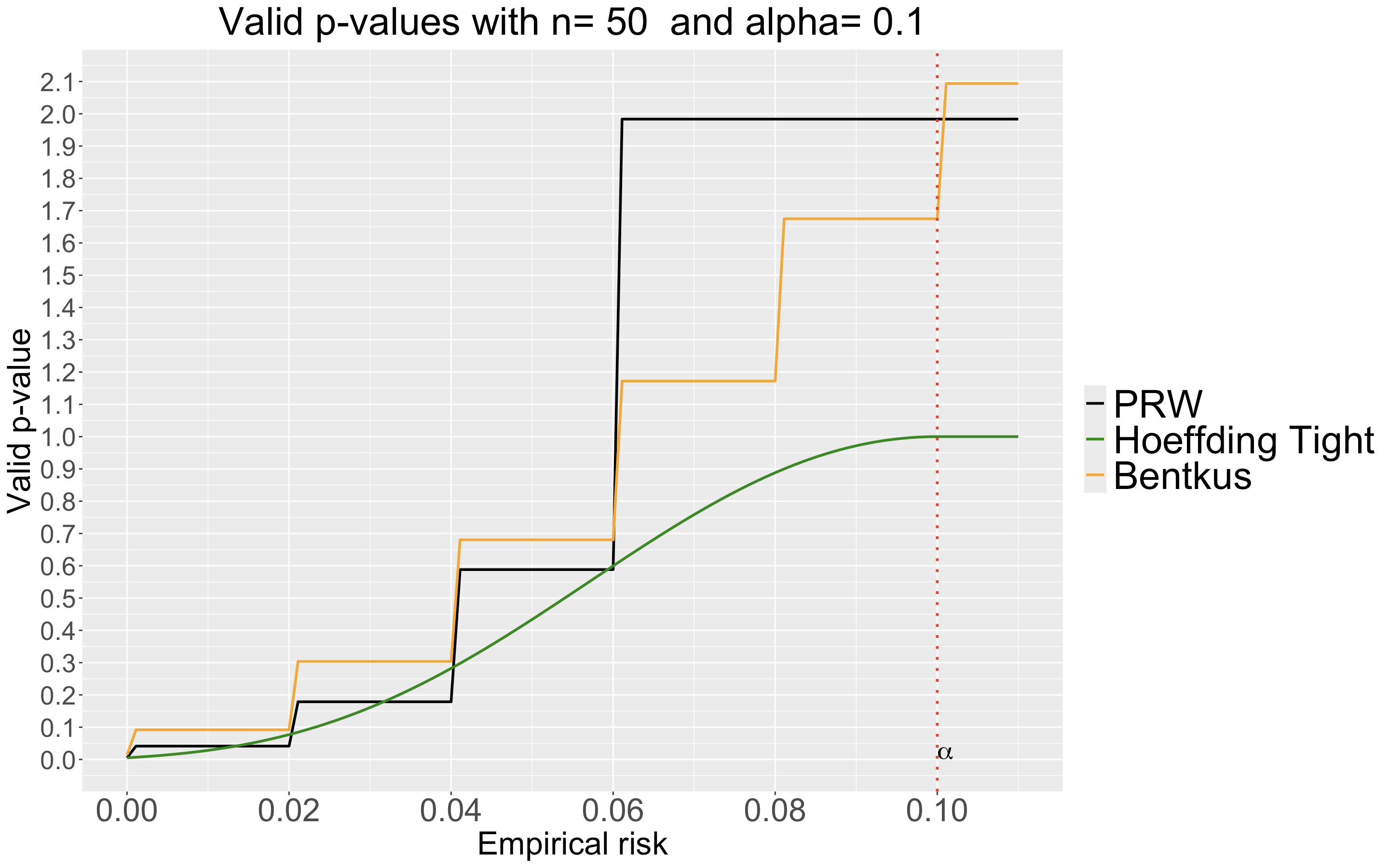}

  \caption{In the same setting as in \ref{hyp}, we consider the following valid p-values as presented in \cite{angelopoulos2022learn} and \cite{bates2021distributionfree}: $p_{Bent}\coloneqq e\Prob\{Bin(n,\alpha)\leq \lceil n\hat{R}\rceil\}$, Bentkus' valid p-value, $p_{HT}\coloneqq exp\Big\{-n\Big(\text{min}\{\hat{R}, \alpha\}log(\frac{\text{min}\{\hat{R}, \alpha\}}{\alpha})+ (1-\text{min}\{\hat{R}, \alpha\})log(\frac{(1-\text{min}\{\hat{R}, \alpha\})}{1-\alpha}) \Big) \Big\}$ Hoeffding's valid p-value (tight version). We consider the valid p-value introduced in this work: $g\Big(\text{min}\{\hat{R}, \frac{\gamma(\alpha)-1}{n}\};\alpha\Big)$, PRW's valid p-value.}
\end{figure}

\begin{table}[ht]
\centering
\begin{tabular}{rrrrr}
  \hline
 \rowcolor{lightgray} $\hat{R}$ & PRW & Hoeffding Tight & Bentkus \\ 
  \hline
  0.0000 & 0.0000 & 0.0000 & 0.0001 \\ 
    0.0015 & 0.0004 & 0.0001 & 0.0009 \\ 
   0.0030 & 0.0004 & 0.0001 & 0.0009 \\ 
 0.0045 & 0.0004 & 0.0002 & 0.0009 \\ 
 0.0061 & 0.0004 & 0.0003 & 0.0009 \\ 
    0.0076 & 0.0004 & 0.0004 & 0.0009 \\ 
 0.0091 & 0.0004 & 0.0006 & 0.0009 \\ 
    0.0106 & 0.0024 & 0.0009 & 0.0053 \\ 
    0.0121 & 0.0024 & 0.0013 & 0.0053 \\ 
    0.0136 & 0.0024 & 0.0018 & 0.0053 \\ 
    0.0152 & 0.0024 & 0.0024 & 0.0053 \\ 
    0.0167 & 0.0024 & 0.0033 & 0.0053 \\ 
    0.0182 & 0.0024 & 0.0043 & 0.0053 \\ 
    0.0197 & 0.0024 & 0.0056 & 0.0053 \\ 
    0.0212 & 0.0109 & 0.0073 & 0.0213 \\ 
    0.0227 & 0.0109 & 0.0093 & 0.0213 \\ 
    0.0242 & 0.0109 & 0.0117 & 0.0213 \\ 
    0.0258 & 0.0109 & 0.0146 & 0.0213 \\ 
    0.0273 & 0.0109 & 0.0180 & 0.0213 \\ 
    0.0288 & 0.0109 & 0.0221 & 0.0213 \\ 
    0.0303 & 0.0379 & 0.0269 & 0.0645 \\ 
    0.0318 & 0.0379 & 0.0325 & 0.0645 \\ 
    0.0333 & 0.0379 & 0.0389 & 0.0645 \\ 
    0.0348 & 0.0379 & 0.0463 & 0.0645 \\ 
    0.0364 & 0.0379 & 0.0548 & 0.0645 \\ 
    0.0379 & 0.0379 & 0.0643 & 0.0645 \\ 
    0.0394 & 0.0379 & 0.0750 & 0.0645 \\ 
    0.0409 & 0.1094 & 0.0869 & 0.1565 \\ 
    0.0424 & 0.1094 & 0.1002 & 0.1565 \\ 
    0.0439 & 0.1094 & 0.1149 & 0.1565 \\ 
      0.0455 & 0.1094 & 0.1310 & 0.1565 \\ 
    0.0470 & 0.1094 & 0.1485 & 0.1565 \\ 
    0.0485 & 0.1094 & 0.1676 & 0.1565 \\ 
    0.0500 & 0.2753 & 0.1881 & 0.3185 \\ 
    0.0515 & 0.2753 & 0.2102 & 0.3185 \\ 
    0.0530 & 0.2753 & 0.2337 & 0.3185 \\ 
    0.0545 & 0.2753 & 0.2587 & 0.3185 \\ 
    0.0561 & 0.2753 & 0.2851 & 0.3185 \\ 
    0.0576 & 0.2753 & 0.3128 & 0.3185 \\ 
    0.0591 & 0.2753 & 0.3417 & 0.3185 \\ 
    0.0606 & 0.6388 & 0.3718 & 0.5601 \\ 
    0.0621 & 0.6388 & 0.4029 & 0.5601 \\ 
    0.0636 & 0.6388 & 0.4349 & 0.5601 \\ 
    0.0652 & 0.6388 & 0.4677 & 0.5601 \\ 
    0.0667 & 0.6388 & 0.5010 & 0.5601 \\ 
   \hline
\end{tabular}
\caption{Values of the p-values across different values of the empirical risk, using $\alpha=0.1, n=100$. Numbers are rounded to 4 decimals.}
\end{table}

\end{document}